\documentclass[11pt]{article}

\usepackage[margin=1in]{geometry}

\usepackage{authblk}
\usepackage{microtype}
\usepackage{graphicx}
\usepackage{subcaption}
\usepackage{algorithm}
\usepackage{algpseudocode}
\usepackage{booktabs} 
\usepackage{makecell}
\usepackage{wrapfig}
\usepackage{diagbox}
\usepackage{enumitem}%
\usepackage{amsfonts}
\usepackage{xcolor}

\usepackage{amsthm}
\usepackage{cprotect}
\usepackage{hyperref}
\usepackage{amsmath,amssymb}
\usepackage{multirow}
\usepackage{fancyhdr}
\renewcommand{\headrulewidth}{0pt}
\fancypagestyle{firststyle}
{
   \fancyhf{}
   \fancyfoot[R]{\thepage}
   \renewcommand{\headrulewidth}{0pt} 
}
\usepackage{fancyvrb}
\usepackage{natbib}
\usepackage{mathtools}

\usepackage{dsfont}
\newtheorem{theorem}{Theorem}[section]
\newtheorem{lemma}[theorem]{Lemma}

\newtheorem{corollary}[theorem]{Corollary}

\newtheorem{assumption}[theorem]{Assumption}

\newcommand{\R}{\mathbb{R}}

\usepackage{makecell}

\allowdisplaybreaks
\title{Spatial Conformal Inference through Localized Quantile Regression}
\author{Hanyang Jiang, Yao Xie}
\date{Dec 2023}

\begin{document}

\maketitle

\begin{abstract}
Reliable uncertainty quantification at unobserved spatial locations, particularly for complex and heterogeneous datasets, is a key challenge in spatial statistics. Traditional methods like Kriging rely on assumptions such as normality, which often fail in large-scale datasets, leading to unreliable intervals. Machine learning methods focus on point predictions but lack robust uncertainty quantification. Conformal prediction provides distribution-free prediction intervals but often assumes i.i.d. data, which is unrealistic for spatial settings. We propose Localized Spatial Conformal Prediction (LSCP), designed specifically for spatial data. LSCP uses localized quantile regression and avoids i.i.d. assumptions, instead relying on stationarity and spatial mixing, with finite-sample and asymptotic conditional coverage guarantees. Experiments on synthetic and real-world data show that LSCP achieves accurate coverage with tighter, more consistent prediction intervals than existing methods.
\end{abstract}

\section{Introduction}

Quantifying uncertainty at unobserved spatial locations has been a longstanding challenge in spatial statistics, particularly in practical applications such as weather forecasting \citep{siddique2022survey} and mobile signal coverage estimation \citep{jiang2024learning}. Traditional methods like Kriging rely on strong parametric assumptions, including normality and stationarity, to model spatial relationships and quantify uncertainty \citep{cressie2015statistics}. However, the failure of these assumptions in the complex spatial datasets \citep{heaton2019case} results in unreliable uncertainty quantification. The issue is especially pronounced when constructing prediction intervals, as deviations from stationarity or Gaussianity can critically undermine their validity \citep{fuglstad2015does}.

While many methods have been developed to handle heterogeneity \citep{gelfand2005bayesian, duan2007generalized}, these approaches are often computationally expensive and cannot scale effectively for massive datasets. Furthermore, fully modeling the underlying process is not always necessary, particularly when the goal is to produce reliable prediction intervals. Recently, machine learning approaches have offered alternative strategies for spatial prediction \citep{hengl2018random, chen2020deepkriging}, though they tend to focus on point predictions and often lack rigorous uncertainty quantification.

Conformal prediction, introduced by \citep{vovk2005algorithmic}, provides a powerful, distribution-free approach to uncertainty quantification. Its ability to generate valid prediction sets without assumption on the underlying data distribution and the prediction model has gained widespread popularity in both machine learning and statistics \citep{lei2014distribution, angelopoulos2023conformal}. By leveraging only the exchangeability of data, conformal prediction ensures valid coverage at any significance level, making it highly attractive for scenarios where a black-box model is used, or traditional parametric assumptions may fail.

However, in many real-world datasets---such as time-series data---the assumption of exchangeability does not hold. To address this, recent work has extended conformal prediction to handle non-exchangeable data. For instance, \citep{tibshirani2019conformal} introduced weighted quantiles to maintain valid coverage in the presence of distributional shifts between training and test sets by leveraging the likelihood ratio between distributions. More recently, \citep{barber2023conformal} tackled the challenge of distribution shifts by bounding the coverage gap using the total variation distance, although the issue of optimizing these weights remains open. For time-series data, further improvements have been made in tightening prediction intervals and building theoretical guarantees, as demonstrated by \citep{xu2023sequential, xu2024conformal}.

In this paper, we extend conformal prediction to spatial data, where the assumption of exchangeability rarely holds. While time-series data can be viewed as a special case of spatial data defined in a one-dimensional time domain, spatial data is inherently multidimensional and poses unique challenges. For example, while time indices are typically discrete and naturally ordered, spatial locations are continuous and lack intrinsic ordering. Despite the prevalence of spatial data in real-world applications, there has been limited work on conformal prediction methods tailored to this context. To address this gap, we propose Localized Spatial Conformal Prediction (LSCP), a novel conformal prediction method that employs localized quantile regression for constructing prediction intervals. Our method and theoretical framework can also be extended to spatio-temporal settings, broadening its applicability. Here is a revised version of the summary:
\begin{itemize}
\item {\it Localized Spatial Conformal Prediction (LSCP)}: We introduce LSCP, a conformal prediction algorithm specifically designed for spatial data, which utilizes localized quantile regression to construct prediction intervals.  
\item {\it Theoretical guarantees:} We establish a finite-sample bound for the coverage gap and provide asymptotic convergence guarantees for LSCP, without requiring the exchangeability of the data.  
\item {\it Numerical Evaluation:} We extensively evaluate LSCP against state-of-the-art conformal prediction methods using both synthetic and real-world datasets. The results highlight LSCP's ability to achieve tighter prediction intervals with valid coverage and more consistent performance across the spatial domain. \end{itemize}


\subsection{Literature}

\paragraph{\it Conformal prediction beyond exchangeability.}
Traditional conformal prediction relies on the assumption of i.i.d. or exchangeability, which is often difficult to satisfy in real-world datasets. Recent research has focused on extending methods and theoretical frameworks to non-exchangeable data to broaden applicability. A prominent approach is weighted conformal prediction, which assigns greater importance to samples deemed more "reliable." The work from \citep{tibshirani2019conformal} explored the covariate shift scenario, where the distribution of features $X$ differs between calibration and test data, while the conditional distribution $Y|X$ remains unchanged. They demonstrated that weighting samples by the ratio of the distributions restores exchangeability. Building on this, \citep{barber2023conformal} proposed a general weighted conformal framework and provided an analysis of the coverage gap in the general non-exchangeable setting, offering insights into the relationship between the weights and the coverage gap.

\paragraph{\it Conformal prediction for time series.}
One key application of conformal prediction in non-exchangeable settings is time-series data. Some studies \citep{gibbs2021adaptive, zaffran2022adaptive} have focused on adaptively adjusting the significance level $\alpha$ over time to achieve valid coverage. Another work \citep{angelopoulos2024conformal}, built upon the idea of control theory, prospectively model the non-conformity scores in an online setting. Another line of research \citep{tibshirani2019conformal, xu2024conformal} follows the concept of weighted conformal prediction by assigning higher weights to more recent data points. The choice of the weights plays a critical role in determining the empirical performance of the method. However, as noted by \citep{barber2023conformal}, no universally optimal weighting strategy has been found, leaving room for further exploration and optimization in this area. Without exchangeability in the time series setting, the finite-sample guarantee does not necessarily hold, and asymptotic coverage can be achieved instead with certain additional assumptions. 

\paragraph{\it Conformal prediction for spatial data.} The spatial setting represents a broader and more complex domain compared to time-series data, yet research on conformal prediction for spatial contexts remains limited. A recent study by \citep{mao2024valid} introduced a spatial conformal prediction method under the infill sampling framework, where data density increases within a bounded region. The key finding is that for spatial data, the exact or asymptotic exchangeability holds in certain settings. The algorithm falls within the category of weighted conformal prediction, employing kernel functions such as the radial basis function (RBF) kernel as weights. Similarly, \citep{guan2023localized} proposed a general localized framework for conformal prediction. While not explicitly designed for spatial data, their method also advocates using kernel functions for weighting, highlighting its potential relevance in spatial applications.

\section{Problem setup}
 In this paper, we consider a spatial setting with observations \(\{Z(s_i)\}_{i=1}^n\), where \(Z(s)= (X(s_i), Y(s_i))\) represents a random field observed at a finite set of spatial locations \(s_i\). Here, \(Y(s) \in \mathbb{R}\) denotes the response variable, and \(X(s) \in \mathbb{R}^p\) is the associated feature vector. The feature vector \(X(s)\) can include any relevant information that is useful for predicting \(Y(s)\), such as spatial location $s$. Unlike the time series setting, where observations are indexed by fixed times \(t\), the spatial setting considers the locations \(s\) as random variables sampled from a distribution \(g(s)\), while allowing for potential dependency within the random field \(Z(s)\).

In the context of conformal prediction, the objective is to construct a prediction region \(\hat{C}_n(X(s_{n+1}))\) for an unobserved response \(Y(s_{n+1})\) given a known feature vector \(X(s_{n+1})\). For a user-specified confidence level \(\alpha\), we aim to ensure that the probability of \(Y(s_{n+1})\) falling within the prediction region exceeds \(1 - \alpha\). This notion of coverage can be interpreted in two ways: marginal coverage and conditional coverage. 
\emph{Marginal coverage} is defined as
\[
\mathbb{P}(Y(s_{n+1}) \in \hat{C}_n(X(s_{n+1}))) \geq 1 - \alpha,
\]
whereas \emph{conditional coverage} requires that
\[
\mathbb{P}(Y(s_{n+1}) \in \hat{C}_n(X(s_{n+1})) \mid X(s_{n+1})) \geq 1 - \alpha.
\]
Conditional coverage is a stronger condition than marginal coverage as it requires valid coverage for different $X(s)$. However, as shown by \citep{foygel2021limits}, achieving conditional coverage universally is impossible without making additional assumptions about the data distribution. In traditional conformal prediction settings, where data points are assumed to be i.i.d. or exchangeable, only marginal coverage is typically guaranteed. Besides valid coverage, constructing a prediction region that is as narrow as possible is desirable to improve the empirical performance of conformal prediction. The whole space is an example of a trivial prediction set, but it is too large to be useful and informative.

\section{Method}
\subsection{Background: Conformal prediction}
conformal prediction is a widely used technique for constructing prediction intervals with finite-sample guarantees under minimal assumptions. The method is designed to provide a prediction interval for a response variable \( Y_{n+1} \) associated with a new feature vector \( X_{n+1} \), given a predictive model \( \hat{f} \) and a set of observations \( \{(X_i, Y_i)\}_{i=1}^n \). The data is divided into two disjoint sets: a training set used to fit the model \( \hat{f} \), and a calibration set used to calculate non-conformity scores, which quantify the uncertainty of the model at each point. A common choice for non-conformity score is \( \hat{\varepsilon}_i = |Y_i - \hat{f}(X_i)| \), which measures how well the model prediction aligns with the observed response. There is no restriction on the choice of the non-conformity score.

To construct a prediction interval for \( Y_{n+1} \), the empirical quantile of the non-conformity scores $\hat{Q}_n$ from the calibration set is used as the estimate. Specifically, for a user-specified confidence level \( 1 - \alpha \), the prediction interval is defined as
\[
\hat{C}_{n}(X_{n+1}) = \left[ \hat{f}(X_{n+1}) - \hat{Q}_n(1-\alpha), \; \hat{f}(X_{n+1}) + \hat{Q}_n(1 - \alpha) \right].
\]
 The empirical quantile can be written explicitly as
 \[
 \hat{Q}_n(p) = \inf\{e\in\mathbb{R}:\sum_{i=1}^n \frac{1}{n}\mathbf{1}\{\hat{\varepsilon}(s_i)\le e\}\le p\}.
 \]
 Utilizing the exchangeability of the data, the method provides a prediction interval with valid marginal coverage, ensuring that \( Y_{n+1} \) falls within the interval with probability at least \( 1 - \alpha \). Besides the strong theoretical guarantee, conformal prediction is appealing due to its flexibility in the prediction model, as it has no assumptions about the underlying distribution of \( Y \) or the form of the model \( \hat{f} \).

Classical geostatistical methods like Gaussian Process (GP) assume that any finite set of observations has a joint Gaussian distribution. Given data, it provides both mean prediction and uncertainty estimates for unseen points, making it suitable for regression tasks where probabilistic predictions are desirable. In contrast to conformal prediction, which constructs prediction intervals by assessing model residuals without assumptions on data distribution, GPs rely on a Gaussian prior and explicit covariance structure. While conformal prediction provides finite-sample coverage guarantees under minimal assumptions, GPs often lose coverage when data does not satisfy the Gaussian assumption. Furthermore, it is computationally expensive to learn a GP model, which makes it difficult for GP to utilize as much data as possible. This also weakens the performance of GP. On the contrary, conformal prediction is computationally efficient and can be used with any prediction model.

\begin{figure}[t]
    \centering
    \includegraphics[width=0.8\linewidth]{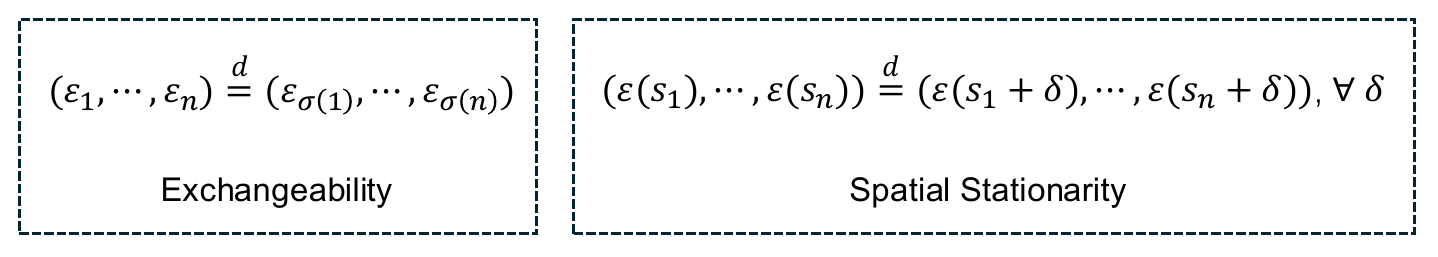}
    \caption{Difference between exchangeability and spatial stationarity.}
    \label{fig:stationarity}
\end{figure}

\subsection{Proposed method: Local spatial conformal prediction (LSCP)}
In spatial settings, data often exhibit significant dependence across locations, and taking the spatial dependence into account can improve the accuracy of prediction intervals. To account for this, it is advantageous to base predictions on nearby data points, as spatially proximate observations are likely to share similar distributions and therefore provide more reliable information. The recent study by \citep{barber2023conformal} highlights the importance of weighting calibration data differently, depending on their relevance to the target prediction point.

To construct the prediction interval, we first split the dataset into a training set and a calibration set. The training set is used to train the prediction model $\hat{f}$, while the calibration set decides the width of the interval. We assume the calibration set consists of observations $(X(s_1), Y(s_1)), \dots, (X(s_n), Y(s_n))$. For a new observation $(X(s_{n+1}), Y(s_{n+1}))$, the aim is to construct a prediction interval by selecting a neighborhood of data from the calibration data. Here we use $N(s_{n+1})$ to represent all the indices $i$ so that $s_i$ is in the neighborhood of $s_{n+1}$. The neighborhood can be determined via various criteria, and a common approach is to include all nearby points located within a specified distance threshold. In the paper, we use $k$-nearest neighbor for simplicity.

Given the trained model $\hat{f}$, the non-conformity scores are defined as $\hat{\varepsilon}(s_i) = Y(s_i) - \hat{f}(X(s_i))$. We then define the conditional cumulative distribution function (CDF) for the non-conformity scores based on the selected neighbors, denoted by 
\[
F(e | \{\hat{\varepsilon}(s_i)\}_{i\in N(s_{n+1})}) = \mathbb{P}(\hat{\varepsilon}(s_{n+1}) \leq e | \{\hat{\varepsilon}(s_i)\}_{i\in N(s_{n+1})}).
\]
The conditional quantile \( Q_n(p) \) is defined as
\[
Q_n(p) = \inf \{ e \in \mathbb{R} : F(e | \{\hat{\varepsilon}(s_i)\}_{i\in N(s_{n+1})}) \geq p \}.
\]
The empirical quantile places an equal weight on all the points, which may not fully capture the dependence structure. Instead, we apply a quantile regression estimator \( \widehat{Q}_n \) on the residuals \( \{\hat{\varepsilon}(s_i)\}_{i\in N(s_{n+1})} \). The estimator, \( \widehat{Q}_n(\alpha) \), predicts the \(\alpha\)-quantile of the residual \( \hat{\varepsilon}(s_{n+1}) \) given the values of its neighbors. For computational efficiency, we use Quantile Random Forests from \citep{meinshausen2006quantile}, although other quantile regression techniques could also be applied. The resulting prediction interval is:
\begin{equation*}
\begin{aligned}
\hat{C}_{n}(X(s_{n+1})) = [ \hat{f}(X(s_{n+1})) + \widehat{Q}_{n}(\beta^*), \; \hat{f}(X(s_{n+1})) + \widehat{Q}_{n}(1 - \alpha + \beta^*) ],      
\end{aligned}  
\end{equation*}
where \( \beta^* = \operatorname{argmin}_{\beta \in [0, \alpha]}(\widehat{Q}_{n}(1 - \alpha + \beta) - \widehat{Q}_{n}(\beta)) \). Here $\beta$ is optimized to find the tightest interval.

To formalize this approach, let \( \tilde{X}(s_{i}) = (\hat{\varepsilon}(s_j))_{j\in N(s_{i})} \) and \( \tilde{Y}(s_i) = \hat{\varepsilon}(s_i) \). The quantile regression, which learns the conditional quantile of \( F(\tilde{Y}(s_{n+1}) | \tilde{X}(s_{n+1})) \), computes:
\[
\widehat{Q}_n(p) = \inf \{ e \in \mathbb{R} : \sum_{i\in N(s_{n+1})} \omega_i \mathbf{1}\{\tilde{Y}(s_i) \leq e\} \leq p \},
\]
where the weights \( \omega_i \) are learned through quantile regression.

\begin{figure}[t]
    \centering
    \includegraphics[width=0.7\linewidth]{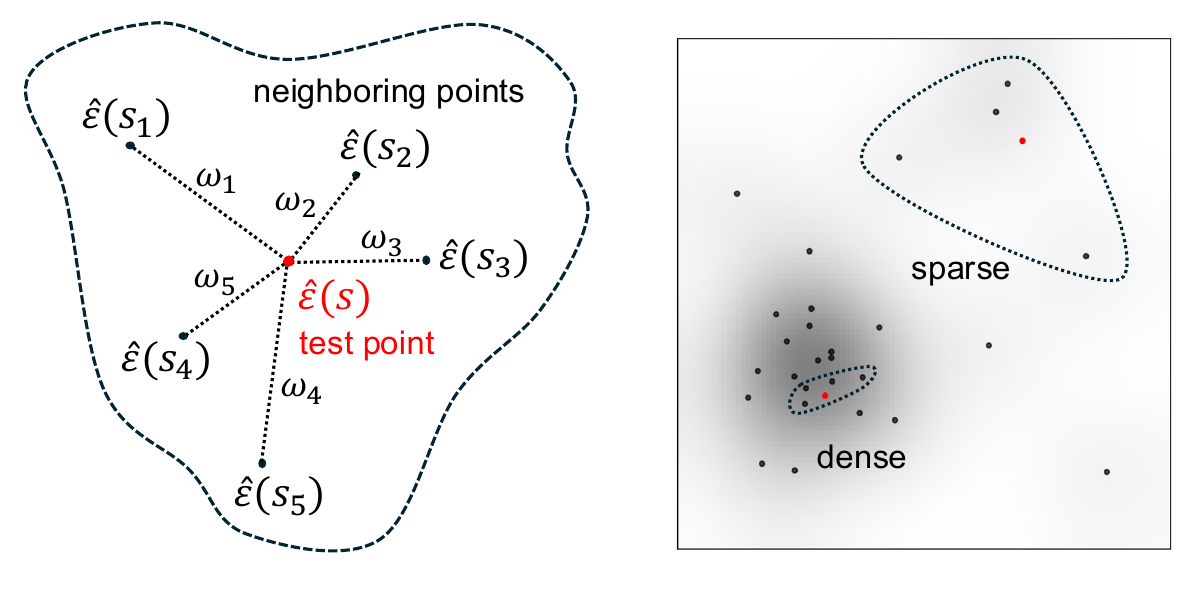}
    \caption{
Illustration of Localized Spatial Conformal Prediction (LSCP). Left: $k$-nearest neighbors with weights predict test points and quantify uncertainty. Right: Neighborhoods adapt to dense and sparse regions. Red points mark test locations, dashed circles show 5-nearest neighbors, and darker grey indicates lower uncertainty. LSCP uses neighbors to construct prediction intervals.}
    \label{fig:WQR}
\end{figure}

\subsection{Comparison with related methods}
\subsubsection{Global Spatial Conformal Prediction}
Global Spatial Conformal Prediction (GSCP), introduced in \citep{mao2024valid}, applies equal weighting to all non-conformity scores across the calibration dataset, where the non-conformity score is defined as
\[
\hat{\varepsilon}(s_i) = \left|\frac{Y_i - \hat{f}(s_i, X(s_i))}{\hat{\sigma}(s_i, X(s_i))}\right|.
\]
The estimated quantile at any point is given by
\begin{equation}
\label{eq_gscp}
\hat{Q}_n(p) = \inf\{e \in \mathbb{R} : \frac{1}{n} \sum_{i=1}^n  \mathbf{1}\{\hat{\varepsilon}(s_i) \le e\} \le p\}.
\end{equation}

The method can be viewed as an extension of split conformal prediction. Its primary advantage is the strong theoretical guarantee, achieved with minimal assumptions. When the spatial locations are sampled i.i.d. from a common distribution \( F_\varepsilon(s) \), the data becomes exchangeable because switching the order does not change the joint probability. The marginal coverage automatically holds in this situation.

However, in real-world applications, the performance of global conformal methods like GSCP may be limited because data distributions can vary across locations. Unfortunately, the exchangeability does not hold for localized methods. Since GSCP lacks adaptivity to different locations, it tends to construct intervals that may be overly conservative to meet the marginal coverage requirement, leading to overcoverage in some regions.

\begin{algorithm}[!t]
\caption{Spatial Conformal Prediction}
\label{alg1}
\begin{algorithmic}[1] 
\Require Dataset $\left\{\left(x(s_i), y(s_i)\right)\right\}_{i=1}^n$, prediction algorithm $\mathcal{A}$, significance level $\alpha$
\Ensure{Prediction intervals $\{\widehat{C}_{n}\left(x(s_{n+1})\right)\}$}
\State Split the dataset into training data and calibration data.
\State Train the prediction model $\hat{f}$ with the training data using the prediction algorithm $\mathcal{A}$.
\State Select the neighborhood of $s_{n+1}$ in the calibration data, denoted as $N(s_{n+1})$.
\State Compute the non-conformity scores $\hat{\varepsilon}(s_i) = Y(s_i) - \hat{f}(X(s_i))$ for all data in the calibration set.
\State Set $\tilde{Y}(s_i) = \hat{\varepsilon}(s_i)$ and\\ $\tilde{X}(s_i) = (\hat{\varepsilon}(s_{j_1}), \dots, \hat{\varepsilon}(s_{j_{|N(s_i)|}}))$, where $s_{j} \in N(s_i)$.
\State Fit quantile regression $\widehat{Q}_n$ with all pairs $(\tilde{X}(s_i), \tilde{Y}(s_i))$ in the calibration data.
\State Obtain the prediction interval $\hat{C}_{n}(X(s_{n+1}))$.
\end{algorithmic}
\end{algorithm}

\subsubsection{Smoothed Local Spatial CP}
Smoothed Local Spatial Conformal Prediction (SLSCP) \citep{mao2024valid} improves upon previous GSCP by utilizing only nearby data to construct prediction intervals. The non-conformity scores are defined as in Equation \ref{eq_gscp}, but the quantile is estimated locally:
\begin{equation}
\label{eq_SLSCP}
\hat{Q}_n(p) = \inf \{ e \in \mathbb{R} : \sum_{i \in N(s_{n+1})} \omega_i \mathbf{1}\{\tilde{Y}(s_i) \leq e\} \leq p \}.
\end{equation}
A key difference between our method and SLSCP is in the choice of weights $\omega_i$. In SLSCP, $\omega_i \propto k(\|s_i - s_{n+1}\|)$, where $k$ is a kernel function that depends solely on the distance between spatial locations. In contrast, our method learns $\omega_i$ through quantile regression based on features $\tilde{X}(s)$. This feature vector can contain additional information beyond spatial distance, allowing the learned weights to be more expressive and adaptive than a given kernel function. The numerical results in Section \ref{num_result} demonstrate that our method significantly outperforms SLSCP in experiments.

Another key distinction lies in the theoretical foundations. Conditional on the location \( s_{n+1} \), the asymptotic coverage of SLSCP depends on an infill sampling assumption, where data points become infinitely dense around \( s_{n+1} \), requiring infinitely close data. Additionally, SLSCP assumes the data process to be a composition of an \( L_2 \)-continuous spatial process and a locally i.i.d. noise process \( \varepsilon(s) \). Under these assumptions, the process can be shown to be locally asymptotically exchangeable, achieving asymptotic coverage as data points accumulate around the test location.

In contrast, our method establishes a finite-sample bound for the conditional coverage gap of LSCP, which guarantees asymptotic coverage. A key distinction is that our result conditions on the feature \( X(s_{n+1}) \), rather than solely on the location \( s_{n+1} \), making it more general and broadly applicable. We assume an additive data model with a stationary, spatially mixing error process rather than an i.i.d. process. Spatial mixing implies that dependence between data points diminishes with increasing distance. Instead of requiring infinitely close neighbors, we assume that the average dependence decreases as more neighbors are included. The assumption is reasonable in that, given a fixed dataset, selecting more neighbors for prediction increases the neighborhood size, leading to a natural decay in correlation as more distant neighbors are included. Furthermore, our setting can be extended to spatio-temporal settings, which are commonly encountered in real-world applications. In contrast, SLSCP cannot generalize to time-series data or similar scenarios, as the time index cannot become infinitely dense, as required by the infill sampling assumption.

\subsubsection{Localized Conformal Prediction}
Localized Conformal Prediction (LCP), introduced in \citep{guan2023localized}, provides a general framework for localized conformal prediction rather than the spatial setting only. The method combines GSCP and SLSCP in quantile estimation:
\begin{equation}
\label{eq_LCP}
\hat{Q}_n(p) = \inf \{ e \in \mathbb{R} : \sum_{i=1}^n \omega_i \mathbf{1}\{\tilde{Y}(s_i) \leq e\} \leq p \}.
\end{equation}
Similar to GSCP, LCP uses all calibration data for prediction, but like SLSCP, it applies weights to each data point. The weights are defined as $\omega_i \propto k(X(s_i), X(s_{n+1}))$, where $k$ is a user-specified kernel function measuring similarity between features. The choice allows for more flexibility than the location-based weights in SLSCP, as feature-based weights can capture more detailed information. However, LCP still relies on user-specified weights, meaning the chosen kernel function might not fully capture the dependence structure in the data.

In terms of theoretical assumptions, LCP assumes that the data $\{(X_i, Y_i)\}_{i=1}^{n}$ are i.i.d., which ensures finite-sample marginal coverage. The assumption is stronger than all the other methods, which hinders the generality of the result.

\begin{table*}[!t]
\caption{Comparison of assumptions and algorithms of three related localized conformal prediction methods.}\label{CP_table}
\vspace{0.1in}
\resizebox{\linewidth}{!}{
\begin{tabular}{c c c c}
\toprule
 & LSCP (Ours) & SLSCP \citep{mao2024valid} & LCP \citep{guan2023localized} \\\midrule
Algorithm & localized weighted quantile & localized weighted quantile &
global weighted quantile \\ \midrule
Weights  & learned by quantile regression & parametric kernel  for location $s$
& parametric kernel for feature $X(s)$ \\ \midrule
\makecell[c]{Distributional \\assumption} & \makecell[c]{stationary and spatial mixing\\ noise $\varepsilon(s)$}  & locally i.i.d. noise $\varepsilon(s)$ & globally i.i.d. data $(X(s),Y(s))$\\
\midrule
Data model & additive noise &\makecell[c]{continuous mapping from  $L_2$-continuous\\ spatial and noise processes} &  no assumption\\
\bottomrule
\end{tabular}
}
\vspace{-0.1in}
\end{table*}

\section{Theoretical results}
\subsection{Setting}
Suppose the data is denoted by $\{Z(s_i)\}_{i=1}^n$, where $Z(s)=(X(s), Y(s))$, $s\in\R^d$ denotes the spatial location, $X(s)\in\R^d$ is the feature vector and $Y(s)\in \R$ represents the univariate response. We assume that $Y(s)$ is generated from a true model with unknown additive noise:
\[
Y(s)=f(X(s))+\varepsilon(s), 
\]
where $f$ is an unknown function and $\varepsilon(s)$ represents the noise process, whose marginal distribution is not necessarily Gaussian. Given a pre-trained prediction model $\hat{f}$, we can compute the non-conformity scores \[\hat{\varepsilon}(s) = Y(s) - \hat{f}(X(s)).\]The estimated conditional distribution function $\widehat{F}(\varepsilon|x)$ is defined as \[\widehat{F}(\varepsilon|x)=\sum_{i=1}^{n}w_t(x)\mathbf{1}(\hat{\varepsilon}({s_i})\le \varepsilon).\] Besides, we define the weighted empirical CDF as \[\widetilde{F}(\varepsilon|x)=\sum_{i=1}^{n}w_t(x)\mathbf{1}(\varepsilon({s_i})\le \varepsilon).\]

\subsection{Preliminary}
\paragraph{\it Stochastic Design: Random observation location.}
The key distinction between time series and spatial data lies in the indexing of observations. For time series data, observations are indexed by a fixed, ordered sequence of time points, denoted as $(X_t, Y_t)$. This inherent ordering imposes a natural temporal structure, preventing data points from being exchangeable. In contrast, spatial data are indexed by locations that can be irregularly distributed across space, with each data point in a random field represented as $(X(s), Y(s))$, where $s$ denotes a spatial location. In stochastic design, the spatial locations $s$ are considered random, following some underlying distribution. If we further assume that the locations of both calibration and test data are independently sampled from the same underlying distribution, then exchangeability holds naturally. In this case, marginal coverage is guaranteed for global conformal prediction that constructs prediction region with the whole calibration dataset, as established in \citep{mao2024valid}. However, this does not apply in time series contexts due to the fixed temporal order, making it impossible to freely exchange data points. In a spatial setting, local conformal methods violate this exchangeability assumption because they apply different weights or restrictions based on proximity. As a result, while global conformal methods achieve exchangeability and marginal coverage, local conformal methods require additional conditions to maintain validity. Interestingly, we can consider time series as a special case of spatial data by restricting the spatial domain to a single dimension and treating time as a scalar spatial index. In this case, each time point can be thought of as a distinct “location” on a one-dimensional grid, with the locations ordered sequentially. This perspective highlights that time series analysis is a subset of spatial analysis, albeit with the added constraint of temporal ordering.

\paragraph{\it Spatial mixing.}
Next, we define the strong mixing co-efficient for random field $Z(\cdot)$. Let $\mathcal{F}_Z(T)=\sigma\langle Z(s): s \in T\rangle$ be the $\sigma$-field generated by the variables $\{Z(s)$ : $s \in T\}, T \subset \mathbb{R}^d$. For any two subsets $T_1$ and $T_2$ of $\mathbb{R}^d$, let $\tilde{\alpha}\left(T_1, T_2\right)=$ $\sup \left\{|\mathbb{P}(A \cap B)-\mathbb{P}A \mathbb{P}B|: A \in \mathcal{F}_Z\left(T_1\right), \quad B \in \mathcal{F}_Z\left(T_2\right)\right\}$, and let $d\left(T_1, T_2\right)=\inf \left\{|x-s|: x \in T_1, \quad s \in T_2\right\}$. For $d=1$, we define the strong mixing co-efficient as
$$
\alpha(a ; b)=\sup \{\tilde{\alpha}((x-b, x],[y, y+b)):-\infty<x+a<y<\infty\}, \quad a>0, b>0 \text {. }
$$

Thus, $\alpha(a ; \infty)$ corresponds to the standard strong mixing co-efficient commonly used in the time series case. To define the strong mixing co-efficient for $d \geq 2$, let $\mathcal{R}_k(b) \equiv\left\{\cup_{i=1}^k D_i: \sum_{i=1}^k\left|D_i\right| \leq b\right\}$ be the collection of all disjoint unions of $k$ cubes $D_1, \ldots, D_k$ in $\mathbb{R}^d, k \geq 1, b>0$. Following the definition in \citep{lahiri2003central}, the strong-mixing coefficient for the r.f. $Z(\cdot)$ for $d \geq 2$ is defined as
$$
\alpha(a ; b)=\sup \left\{\tilde{\alpha}\left(T_1, T_2\right): d\left(T_1, T_2\right) \geq a, T_1, T_2 \in \mathcal{R}_3(b)\right\}.
$$

To simplify the exposition, we further assume that there exists a nonincreasing function $\alpha_1(\cdot)$ with $\lim _{a \rightarrow \infty} \alpha_1(a)=0$ and a nondecreasing function $g(\cdot)$ such that the strong-mixing coefficients $\alpha(a, b)$ satisfies the inequality
\begin{equation}
\alpha(a, b) \leq \alpha_1(a) g(b), \quad a>0, b>0,
\end{equation}
where the function $g(\cdot)$ is bounded for $d=1$ but may be unbounded for $d \geq 2$. Without loss of generality, we may assume that $\alpha_1(\cdot)$ is left continuous and $g(\cdot)$ is right continuous (otherwise, replace $\alpha_1(a)$ by $\alpha_1(a-) \geq \alpha_1(a)$ and $g(b)$ by $g(b+) \geq g(b))$. We shall specify exact conditions on the rate of decay of $\alpha_1(\cdot)$ and the growth rate of $g(\cdot)$ in the statements of the results below.

\subsection{Assumptions}
With the data $(X(s_1),Y(s_1)),\cdots,(X(s_n),Y(s_n))$, we would like to construct a prediction region for $Y(s_{n+1})$ where only the feature $X(s_{n+1})$ is known. Define $\omega_{n}(X(s_i))$ to be the normalized weight over samples $s_1,\cdots,s_n$ so that $\sum_{i=1}^{n}\omega_{n}(X(s_i))=1$. The weight function can be some function that measures the similarity between $X(s)$ and $X(s_{n+1})$.

\begin{assumption}[Weight decay]
\label{aw}
There exists $\gamma>0$ so that the normalized weights satisfy
\begin{equation}
\omega_{n}(X(s_i))=o(n^{-\frac{1+\gamma}{2}}),
\end{equation}
for all $i$; meaning that $M_n=\max_{1\le i\le n}\omega_{n}(X(s_i))=o(n^{-\frac{1+\gamma}{2}})$.
\end{assumption}

The requirement assumes that the normalized weights decay at a rate faster than $n^{-\frac{1}{2}}$. As we can see, $\omega_{n}=\frac{1}{n}$ is a special case that satisfies this condition. The condition is also weaker than the requirement of $\omega_{n}=O(\frac{1}{n})$ in a related study \citep{xu2024conformal}. Besides, the assumption can also be inferred from that of SLSCP \citep{mao2024valid} where a GBF kernel with infinitely close data leads to uniform weights.

\begin{assumption}[Estimation quality]
\label{ae}
There exists a sequence $\{\delta_n\}_{n\ge 1}$ such that
{
  \setlength{\abovedisplayskip}{5pt}
  \setlength{\belowdisplayskip}{5pt}
\begin{equation}
\begin{split}
&\sum_{i=1}^{n}\|\hat{\varepsilon}(s_i)-\varepsilon(s_i)\|^2 \le \frac{\delta_{n}^{2}}{M_n}, \\ 
&\|\hat{\varepsilon}(s_{n+1})-\varepsilon(s_{n+1})\| \le \delta_n.
\end{split}
\end{equation}}
\end{assumption}

The assumption requires that the average prediction error be bounded by a term $\delta_n^2$, which is a weaker requirement than the similar condition for bounding the average prediction error in \citep{xu2021conformal}. The reason lies in that $M_n$ is allowed to decay at a slower rate than $\frac{1}{n}$. Notably, our result of the coverage gap does not require $\delta_n$ to converge to zero. However, there are many cases where $\delta_n$ does indeed approach zero. For instance, extensive research has investigated the prediction error of neural networks. Under certain regularization conditions on $f$, \citep{barron1994approximation} shows that $\delta_n = O\left(\frac{1}{\sqrt{n}}\right)$.

\begin{assumption}[Stationary and spatial mixing]
\label{asm}
The random field $\varepsilon(s)$ is stationary and strongly mixing with coefficient $\alpha$. The strong mixing coefficient can be bounded by $\alpha(a,b)\le \alpha_1(a) g(b)$, where $\alpha_1$ is a nonincreasing function with $\lim_{a\rightarrow\infty}\alpha_1(a)=0$. We assume $E_{d\sim g_n}\alpha_1(d)^2\le \frac{M}{n^2}$ where $g_n(d)$ is distribution of the distance between two sample points $s_i$ and $s_j$ ($1\le i,j\le n$). Besides, $F_\varepsilon(x)$ (the CDF of the true non-conformity score) is assumed to be Lipschitz continuous
with constant $L_{n+1}>0$. 
\end{assumption}

The requirement assumes that the true error $\varepsilon(s)$ is spatially stationary and strongly mixing, a condition weaker than the i.i.d. assumption used in \citep{mao2024valid} for spatial conformal prediction. Under this assumption, the original random field $\{X(s), Y(s)\}$ can still exhibit complex dependencies and be highly non-stationary. Here $b$ is a constant in the definition of spatial mixing, which can be any specified positive integer. The assumption further requires that the expectation of the mixing coefficient $\alpha_1(d)$ decays at a certain rate. The decay is reasonable in that a bigger $n$ implies sampling from a larger area, thereby increasing the average distance between calibration points. The assumption is analogous to a common condition in time series analysis, which requires the sum of strong mixing coefficients to be bounded. To sum up, this assumption indicates that dependence between data points diminishes as distance increases.

\subsection{Coverage guarantee}
With the assumptions, we can show that the distance between the empirical CDF of the residuals $\hat{\varepsilon(s)}$ and the noise $\varepsilon(s)$ can actually be bounded through Lemma \ref{l1}
\begin{lemma}[Distance between the empirical CDF of $\varepsilon$ and $\hat{\varepsilon}$]
\label{l1}
Under Assumption \ref{aw} and \ref{ae},
\begin{equation}
\begin{aligned}
\sup_x |\widehat{F}_{n+1}(y)-\widetilde{F}_{n+1}(y)|\le(L_{n+1}+1) \delta_n+2 \sup _y|\widetilde{F}_{n+1}(y)-F_{\varepsilon}(y)|.    
\end{aligned}
\end{equation}
\end{lemma}

Besides, we can also prove that the empirical CDF of the noise $\varepsilon(s)$ converges to its CDF with a high probability.
\begin{lemma}[Convergence of empirical CDF of $\varepsilon$]
Under Assumptions \ref{aw}-\ref{asm}, with probability higher than $1-(2+M+2\sqrt{M}g(b))n^{-\frac{2\gamma}{3}}(\log_2 n+2)^{\frac{4}{3}}$,
\begin{equation}
\sup_y\left|\widetilde{F}_{n+1}(y)-F_\epsilon(y)\right|\le M_n n^{\frac{1+\gamma}{2}},
\end{equation}    
where $\widetilde{F}_{n+1}(y)=\sum_{i=1}^{n}\omega_{n}(X(s_i))\mathbf{1}(\epsilon(s_i)\le y)$.
\end{lemma}

Our main theorem is the following: Using the previous lemma, we establish the asymptotic convergence of conditional coverage.
\begin{theorem}[Conditional coverage guarantee]
Under Assumption \ref{aw}-\ref{asm}, for any $\alpha\in(0,1)$ and sample size $T$, we have
\begin{equation}
\begin{aligned}
\label{main_thm}
&\left|\mathbb{P}\left(Y(s_{n+1}) \in \widehat{C}_{t-1}\left(X(s_{n+1})\right) \mid X(s_{n+1})\right)-(1-\alpha)\right| \\
&\le 4L_{n+1}\delta_n + 6M_n n^{\frac{1+\gamma}{2}}  + (4 + 2M + 4\sqrt{M}g(b))n^{-\frac{2\gamma}{3}}(\log_2 n + 2)^{\frac{4}{3}}.
\end{aligned}
\end{equation}
\end{theorem}

We can establish the same result for marginal coverage with the tower law property.
\begin{corollary}[Marginal Coverage]
Under Assumption \ref{aw}-\ref{asm}, for any $\alpha\in(0,1)$ and sample size $T$, we have
\begin{equation}
\begin{aligned}
&\left|\mathbb{P}\left(Y(s_{n+1}) \in \widehat{C}_{t-1}\left(X(s_{n+1})\right)\right)-(1-\alpha)\right| \\
&\le 4L_{n+1}\delta_n+ 6M_n n^{\frac{1+\gamma}{2}}+ (4+2M+4\sqrt{M}g(b))n^{-\frac{2\gamma}{3}}(\log_2 n+2)^{\frac{4}{3}}. 
\end{aligned}
\end{equation}
\end{corollary} 

From Inequality \ref{main_thm}, we can see that the order of the coverage bound is controlled by $M_n n^{\frac{1+\gamma}{2}}$ and $n^{-\frac{2\gamma}{3}}(\log_2 n+2)^{\frac{4}{3}}$. The first term is equal to $n^{\frac{\gamma-1}{2}}$ in the special case of $M_n = \frac{1}{n}$ and the second term diminishes when $n$ is large enough because $n$ is of higher order than $\log_2 n$. As long as the estimation gap $\delta_n$ goes to zero when $n$ gets larger, the asymptotic conditional coverage can be inferred from the main theorem.

\begin{figure}[!t]
    \centering

    \begin{minipage}[b]{\linewidth}
        \centering
        \textbf{Scenario 1} 
    \end{minipage}
    \begin{minipage}[b]{0.35\linewidth}
        \centering
        \includegraphics[width=0.8\linewidth]{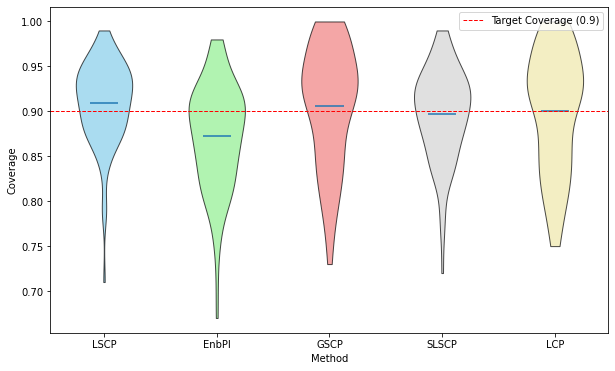}
        \subcaption{Coverage} 
        \label{fig:vc1}
    \end{minipage}
    \hspace{0.1\linewidth}
    \begin{minipage}[b]{0.35\linewidth}
        \centering
        \includegraphics[width=0.8\linewidth]{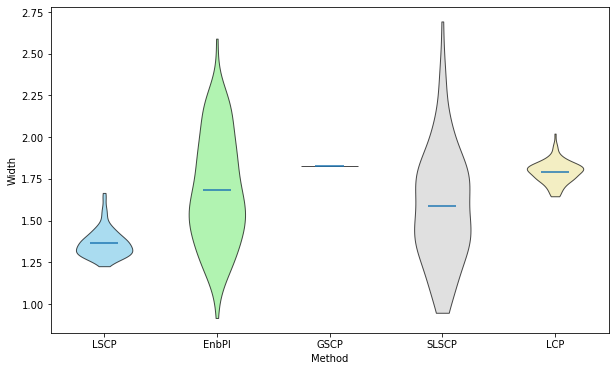}
        \subcaption{Width}
        \label{fig:vw1}
    \end{minipage}

    \begin{minipage}[b]{\linewidth}
        \centering
        \textbf{Scenario 2}
    \end{minipage}
    \begin{minipage}[b]{0.35\linewidth}
        \centering
        \includegraphics[width=0.8\linewidth]{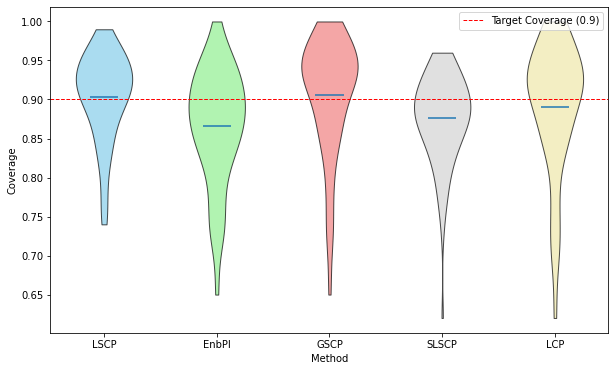}
        \subcaption{Coverage}
        \label{fig:vc2}
    \end{minipage}
    \hspace{0.1\linewidth}
    \begin{minipage}[b]{0.35\linewidth}
        \centering
        \includegraphics[width=0.8\linewidth]{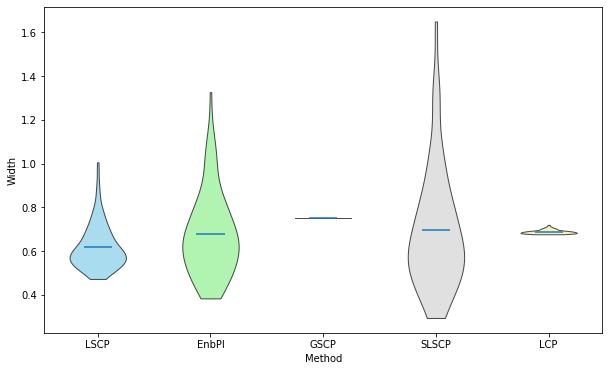}
        \subcaption{Width}
        \label{fig:vw2}
    \end{minipage}

    \begin{minipage}[b]{\linewidth}
        \centering
        \textbf{Scenario 3}
    \end{minipage}
    \begin{minipage}[b]{0.35\linewidth}
        \centering
        \includegraphics[width=0.8\linewidth]{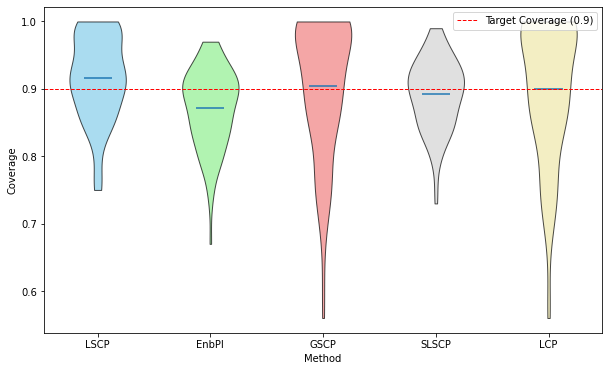}
        \subcaption{Coverage}
        \label{fig:vc3}
    \end{minipage}
    \hspace{0.1\linewidth}
    \begin{minipage}[b]{0.35\linewidth}
        \centering
        \includegraphics[width=0.8\linewidth]{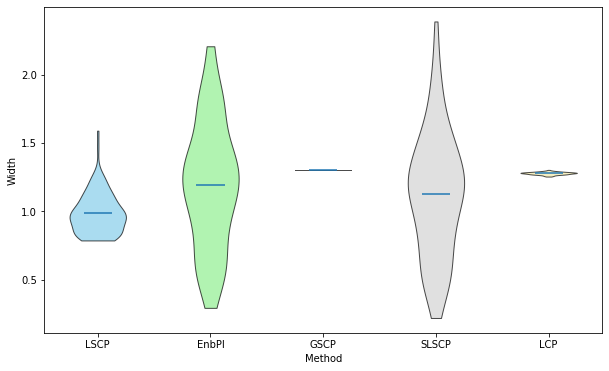}
        \subcaption{Width}
        \label{fig:vw3}
    \end{minipage}

    \begin{minipage}[b]{\linewidth}
        \centering
        \textbf{New Mexico}
    \end{minipage}
    \begin{minipage}[b]{0.35\linewidth}
        \centering
        \includegraphics[width=0.8\linewidth]{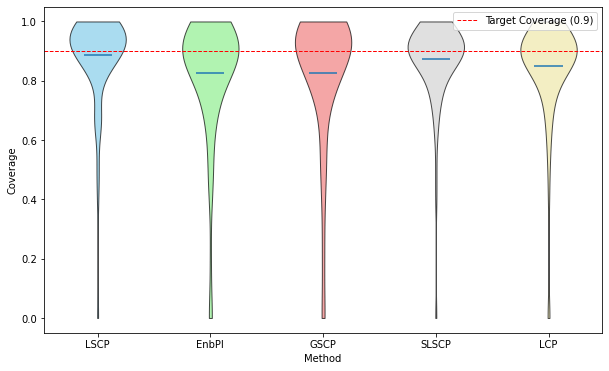}
        \subcaption{Coverage}
        \label{fig:vcNM}
    \end{minipage}
    \hspace{0.1\linewidth}
    \begin{minipage}[b]{0.35\linewidth}
        \centering
        \includegraphics[width=0.8\linewidth]{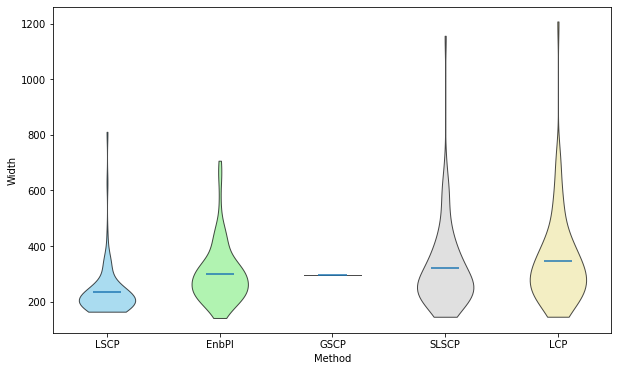}
        \subcaption{Width}
        \label{fig:vwNM}
    \end{minipage}

    \begin{minipage}[b]{\linewidth}
        \centering
        \textbf{Georgia}
    \end{minipage}
    \begin{minipage}[b]{0.35\linewidth}
        \centering
        \includegraphics[width=0.8\linewidth]{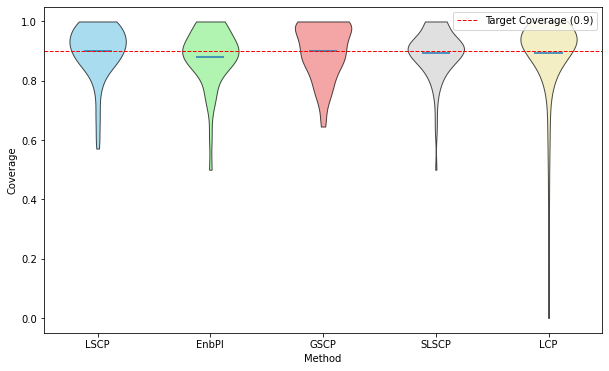}
        \subcaption{Coverage}
        \label{fig:vcGA}
    \end{minipage}
    \hspace{0.1\linewidth}
    \begin{minipage}[b]{0.35\linewidth}
        \centering
        \includegraphics[width=0.8\linewidth]{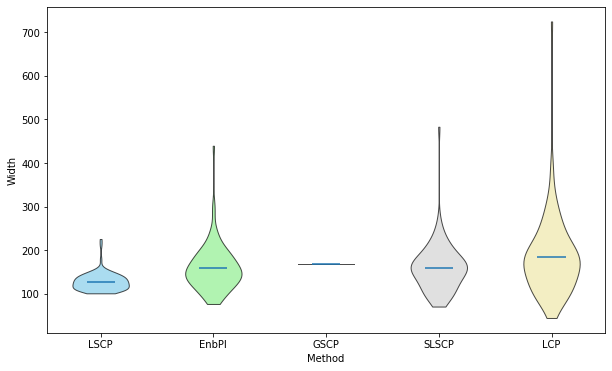}
        \subcaption{Width}
        \label{fig:vwGA}
    \end{minipage}

    \caption{The violin plots on the left show the distribution of coverage across different areas, while the plots on the right show the distribution of width. Each row represents a different scenario or location.}
    \label{fig:violin}
    \vspace{-.1in}
\end{figure}

\begin{figure}[t]
    \centering
    \begin{minipage}[b]{\linewidth}
        \centering
        \title{Scenario 1}
    \end{minipage}
    \begin{minipage}[b]{0.18\linewidth}
        \centering
        \includegraphics[width=\linewidth]{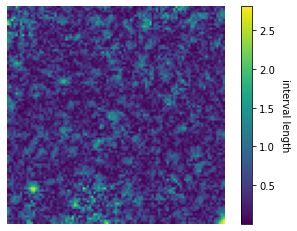}
        \label{fig:res1}
    \end{minipage}
    \begin{minipage}[b]{0.18\linewidth}
        \centering
        \includegraphics[width=\linewidth]{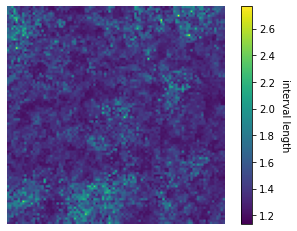}
        \label{fig:LSCP1}
    \end{minipage}
    \begin{minipage}[b]{0.18\linewidth}
        \centering
        \includegraphics[width=\linewidth]{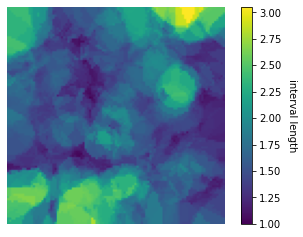}
        \label{fig:Enbpi1}
    \end{minipage}
    \begin{minipage}[b]{0.18\linewidth}
        \centering
        \includegraphics[width=\linewidth]{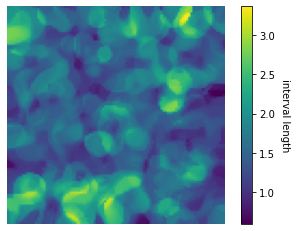}
        \label{fig:SLSCP1}
    \end{minipage}
    \begin{minipage}[b]{0.18\linewidth}
        \centering
        \includegraphics[width=\linewidth]{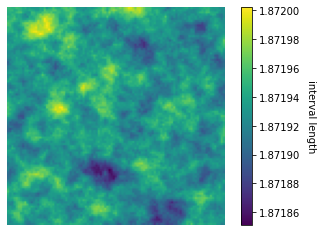}
        \label{fig:LCP1}
    \end{minipage}

    \begin{minipage}[b]{\linewidth}
        \centering
        \title{Scenario 2}
    \end{minipage}

    \begin{minipage}[b]{0.18\linewidth}
        \centering
        \includegraphics[width=\linewidth]{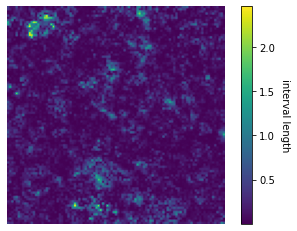}
        \label{fig:res2}
    \end{minipage}
    \begin{minipage}[b]{0.18\linewidth}
        \centering
        \includegraphics[width=\linewidth]{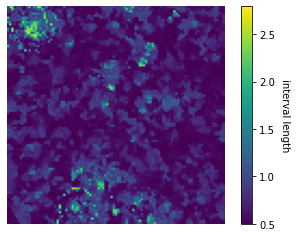}
        \label{fig:LSCP2}
    \end{minipage}
    \begin{minipage}[b]{0.18\linewidth}
        \centering
        \includegraphics[width=\linewidth]{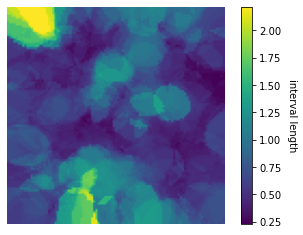}
        \label{fig:Enbpi2}
    \end{minipage}
    \begin{minipage}[b]{0.18\linewidth}
        \centering
        \includegraphics[width=\linewidth]{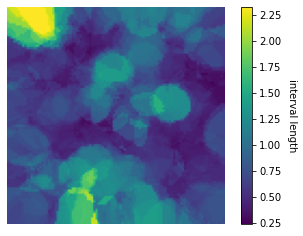}
        \label{fig:SLSCP2}
    \end{minipage}
    \begin{minipage}[b]{0.18\linewidth}
        \centering
        \includegraphics[width=\linewidth]{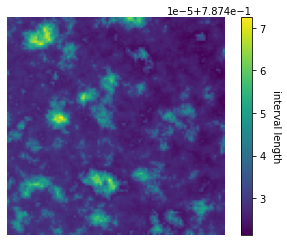}
        \label{fig:LCP2}
    \end{minipage}

    \begin{minipage}[b]{\linewidth}
        \centering
        \title{Scenario 3}
    \end{minipage}

    \begin{minipage}[b]{0.18\linewidth}
        \centering
        \includegraphics[width=\linewidth]{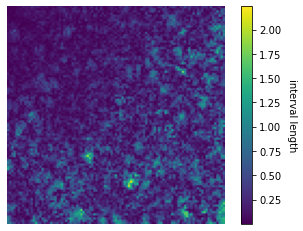}
        \subcaption{Residual}
        \label{fig:res3}
    \end{minipage}
    \begin{minipage}[b]{0.18\linewidth}
        \centering
        \includegraphics[width=\linewidth]{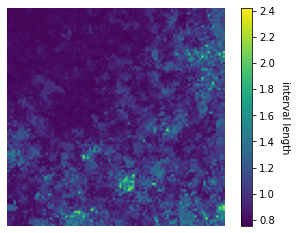}
        \subcaption{LSCP}
        \label{fig:LSCP3}
    \end{minipage}
    \begin{minipage}[b]{0.18\linewidth}
        \centering
        \includegraphics[width=\linewidth]{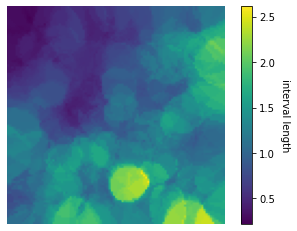}
        \subcaption{EnbPI}
        \label{fig:Enbpi3}
    \end{minipage}
    \begin{minipage}[b]{0.18\linewidth}
        \centering
        \includegraphics[width=\linewidth]{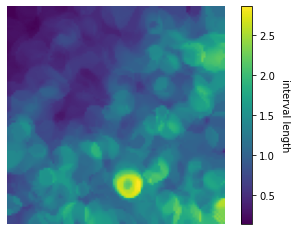}
        \subcaption{SLSCP}
        \label{fig:SLSCP3}
    \end{minipage}
    \begin{minipage}[b]{0.18\linewidth}
        \centering
        \includegraphics[width=\linewidth]{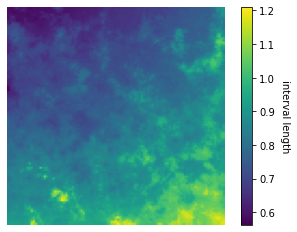}
        \subcaption{LCP}
        \label{fig:LCP3}
    \end{minipage}
    \caption{The heatmaps illustrate the width of the prediction intervals for each method across the three scenarios. The width heatmap of LSCP closely matches the true residual heatmap, demonstrating its ability to capture fine details accurately.}
    \label{fig:combined_scenarios}
    \vspace{-0.2in}
\end{figure}

\section{Experiments}
\label{num_result}

In this section, we compare our proposed LSCP method against four baseline approaches. The first method, EnbPI \citep{xu2021conformal}, is designed for time series data and applies equal weight to the most recent observations. The second method, GSCP \citep{mao2024valid}, uses the entire dataset for prediction, applying equal weights to all data points. The third method, SLSCP \citep{mao2024valid}, leverages k-nearest neighbors for uncertainty quantification, assigning different weights to each point based on the spatial distance between the test data and the calibration data. The fourth method, LCP \citep{guan2023localized}, utilizes all data points for prediction, weighting them according to the similarity between the features of test data and calibration data. Both SLSCP and LCP use the Gaussian kernel as a similarity measure.

We randomly split the dataset into three subsets: $40\%$ for training, $40\%$ for calibration, and $20\%$ for testing. For each method and dataset, the number of neighbors and the Gaussian kernel bandwidth are selected through $5$-fold cross-validation.



\subsection{Synthetic data experiments}

We begin by comparing LSCP with baseline methods across several simulated scenarios. For simplicity, we assume that the locations \( s \) are uniformly sampled from the unit grid \( [0,1] \times [0,1] \). The mean-zero stationary Gaussian process \( X(s) \) is defined using a Matérn covariance function with variance \( \sigma^2 = 1 \), range \( \phi = 0.1 \), and smoothness \( \kappa = 0.7 \). The scenarios are as follows: \begin{enumerate}
    \item \( Y(s) = X(s) + \epsilon(s) \).
    \item \( Y(s) = X(s) | \epsilon(s) | \).
    \item \( Y(s) = X(s) + \sin(\| s \|_2) \epsilon(s) \).
\end{enumerate}

These scenarios incorporate nonlinear and non-stationary settings that extend beyond the assumptions of our theoretical framework. Nevertheless, the empirical results demonstrate that LSCP consistently outperforms other methods across all scenarios.  

In Table \ref{sc123}, LSCP achieves the target coverage of \(90\%\) across all cases while maintaining significantly narrower prediction intervals. Figure \ref{fig:combined_scenarios} visualizes the interval width over the spatial domain for each scenario and the true residuals. GSCP is omitted from these plots as it produces uniform interval widths across space. A larger interval width indicates higher model uncertainty. The results show that LSCP adapts more locally, whereas the other methods exhibit smoother, more uniform patterns. Both EnbPI and SLSCP rely on \( k \)-nearest neighbors for constructing prediction intervals, differing primarily in their weighting schemes.  

To evaluate the methods, we divide the grid into \(10 \times 10\) areas and compute coverage and average interval width for each, as shown in Figure \ref{fig:violin}. Both LSCP and GSCP achieve average coverage above the target \(90\%\), with LSCP exhibiting the most stable and consistent distribution. LSCP also attains the smallest interval width and the highest consistency. LCP performs similarly to GSCP, as both utilize the entire calibration dataset, and as kernel bandwidth increases, LCP weights converge to GSCP’s uniform weights. Overall, LSCP produces tighter valid regions with superior consistency.

\begin{table*}[!t]
    \centering
    \caption{Simulation: The table presents a comparison of the coverage and prediction interval width for five methods across three scenarios. Target coverage is $90\%$.}
    \begin{tabular}{lcccccc}
    \toprule
    Method & S1 Coverage & S1 Width & S2 Coverage & S2 Width & S3 Coverage & S3 Width \\
    \midrule
    LSCP & 92\% & \textbf{1.44} & 90.4\% & \textbf{0.64} & 91.3\% & \textbf{0.97} \\
    EnbPI & 88.5\% & 1.83 & 89.2\% & 0.75 & 86.2\% & 1.12\\
    GSCP & 89.4\% & 1.92 & 89.7\% & 0.77 & 89.6\% & 1.30\\
    SLSCP  & 90.6\%  & 1.64 & 88.2\% & 0.73 & 89.1\% & 1.07 \\
    LCP  & 89.6\% & 1.92 & 94.1\%  & 0.78 & 91.8\%  & 1.43 \\
    \bottomrule
    \end{tabular}
    \label{sc123}
\end{table*}

\begin{table*}[!t]
    \centering
    \caption{Real data: The table presents a comparison of the coverage and prediction interval width for five methods on mobile signal data. The target coverage is $90\%$.}
    \begin{tabular}{lcccc}
    \toprule
    Method & NM Coverage & NM Width & GA coverage & GA width\\
    \midrule
    LSCP & 92.6\% & \textbf{211.3} & 90.6\% & \textbf{130.8}\\
    EnbPI & 88.4\% & 272.2 & 88.6\% & 166.3\\
    GSCP & 90.1\% & 295.9 & 89.2\% & 167.3\\
    SLSCP & 89.8\% & 276.8 & 89.6\% & 167.8\\
    LCP  & 89.7\% & 266.4 & 89.3\% & 166.7\\
    \bottomrule
    \end{tabular}
    \label{table:real_data}
    \vspace{-.1in}
\end{table*}

\subsection{Real data experiments}

In the real-data experiments, we utilize mobile network measurement data from the Ookla public dataset, which records user-reported statewide mobile internet performance. Our analysis focuses on datasets from New Mexico and Georgia. The data is unevenly distributed across urban and suburban areas. We employ kernel regression as the prediction method \(\hat{f}\) and compare the performance of conformal methods. The New Mexico dataset contains 24,983 observations, while the Georgia dataset includes 28,587 data points.  

As shown in Table \ref{table:real_data}, our proposed LSCP method outperforms competing methods by achieving significantly narrower prediction intervals while maintaining better coverage. Similar to the synthetic experiments, we divide each state into \(10 \times 10\) grids and compute coverage and interval width for the test data within each grid. The violin plots in Figure \ref{fig:violin} illustrate the grid-level distribution of these metrics, highlighting spatial performance. Unlike the table, which reports averages over all test points, the plots emphasize spatial consistency, showing higher coverage in urban areas due to denser data.  

The results demonstrate that LSCP consistently achieves high coverage with narrow, uniform intervals, outperforming alternative methods. The violin plots confirm LSCP's stability and uniformity across the spatial domain, underscoring its robustness to non-uniform data distributions. These findings align with our synthetic study, further validating LSCP's effectiveness in diverse and complex settings.




\section{Conclusion}

This paper introduces the Localized Spatial Conformal Prediction (LSCP) method, which addresses key limitations in spatial and spatio-temporal prediction for non-exchangeable data. Traditional methods, such as Kriging, often depend on strong parametric assumptions that may not hold in complex real-world scenarios, including heterogeneous or large-scale spatial datasets. While machine learning techniques offer flexibility and predictive power, they typically lack robust uncertainty quantification, which is critical for applications requiring reliable decision-making.

Our results show that LSCP significantly outperforms existing methods, including GSCP, SLSCP, and EnbPI, by achieving more accurate coverage and narrower prediction intervals. This advantage is particularly pronounced in scenarios with non-stationary and non-Gaussian data, where LSCP’s flexibility enables it to effectively handle such complexities. Moreover, LSCP’s theoretical framework supports extension to spatio-temporal settings where exchangeability is unlikely to hold.

In synthetic experiments, LSCP consistently meets the target coverage with narrower intervals, even in cases beyond its assumptions, demonstrating robustness across diverse scenarios. Real-world experiments further validate LSCP’s utility, showing that it generates detailed uncertainty maps with tighter intervals than baseline methods. These features underscore LSCP’s potential as a reliable and scalable tool for spatial uncertainty quantification, offering both theoretical guarantees and practical performance benefits.

\bibliography{reference}
\bibliographystyle{icml2025}

\newpage
\appendix
\onecolumn
\setcounter{table}{0}
\setcounter{figure}{0}
\setcounter{equation}{0}

\section{Proofs}
In our spatial conformal prediction method, the weight function is defined as $\omega_{n}$. The weighted empirical distribution for the true noise $\varepsilon$ is
\begin{equation}
\tilde{F}_{n+1}(y)=\sum_{i=1}^{n}\omega_{n}(X(s_i))\mathbf{1}(\varepsilon(s_i)\le y).
\end{equation}
Here $\omega_{n}$ is the normalized weight. Besides, we also define the weighted empirical distribution for the residual $\hat{\varepsilon}$ as 
\begin{equation}
\widehat{F}_{n+1}(y)=\sum_{i=1}^{n}\omega_{n}(X(s_i))\mathbf{1}(\hat{\varepsilon}(s_i)\le y).
\end{equation}

We assume an additive true model which is commonly used in literature like \citep{xu2021conformal}:
\begin{equation}
Y(s)=f(X(s))+\varepsilon(s).
\end{equation}
Considering that the residual is $\hat{\varepsilon}(s)=Y(s)-\hat{f}(X(s))$, it follows
\begin{equation}
\varepsilon(s)=\hat{f}(X(s))-f(X(s)),
\end{equation}
which is the prediction error.

The following Lemma bounds the distance between the weighted empirical distribution for the residual and true error.

\begin{lemma}
Under Assumption \ref{ae} and \ref{aw},
\begin{equation}
\sup_x |\widehat{F}_{n+1}(y)-\widetilde{F}_{n+1}(y)|\le(L_{n+1}+1) \delta_n+2 \sup _y|\widetilde{F}_{n+1}(y)-F_{\varepsilon}(y)|.
\end{equation}
\end{lemma}

\begin{proof}
Using Assumption \ref{ae}, we have that
\begin{equation}
\begin{aligned}
\sum_{i=1}^{n}\omega_{n}(X(s_i))|\varepsilon(s_i) - \hat{\varepsilon}(s_i)|\le M_n\sum_{i=1}^{n}|\varepsilon(s_i) - \hat{\varepsilon}(s_i)|\le \delta_n^2.
\end{aligned}
\end{equation}

Let $S=\{t:|\varepsilon(s_i) - \hat{\varepsilon}(s_i)|\ge \delta_n\}$. Then
\begin{equation}
\delta_n\sum_{i\in S}\omega_{n}(X(s_i))\le\sum_{i=1}^{n}\omega_{n}(X(s_i))|\varepsilon(s_i) - \hat{\varepsilon}(s_i)|\le \delta_n^{2}.
\end{equation}
So $\sum_{i\in S}\omega_{n}(X(s_i))\le  \delta_n$. Then
\begin{align*}
 |\widehat{F}_{n+1}(y)-\widetilde{F}_{n+1}(y)| &\le  \sum_{i=1}^n\omega_{n}(X(s_i))|\mathbf{1}\{\hat{\varepsilon}(s_i) \leq y\}-\mathbf{1}\{\varepsilon(s_i) \leq y\}| \\
&\le \sum_{i=1}^n\omega_{n}(X(s_i))+\sum_{i \notin S}\omega_{n}(X(s_i))|\mathbf{1}\{\hat{\varepsilon}(s_i) \leq y\}-\mathbf{1}\{\varepsilon(s_i) \leq y\}| \\
&\stackrel{(i)}{\leq} \sum_{i=1}^n\omega_{n}(X(s_i))+\sum_{i \notin S} \omega_{n}(X(s_i))\mathbf{1}\{|\varepsilon(s_i)-y| \leq \delta_n\} \\
&\le \sum_{i=1}^n\omega_{n}(X(s_i))+\sum_{i=1}^n \omega_{n}(X(s_i))\mathbf{1}\{|\varepsilon(s_i)-y| \leq \delta_n\} \\
&\le \delta_n+\mathbb{P}(|\varepsilon(s_{n+1})-y| \leq \delta_n) + \\
& \quad \sup_y\left| \sum_{i=1}^n \omega_{n}(X(s_i))\mathbf{1}\{|\varepsilon(s_i)-y| \leq \delta_n\}-\mathbb{P}(|\varepsilon(s_{n+1})-y| \leq \delta_n)\right| \\
&= \delta_n+[F_{\varepsilon}(y+\delta_n)-F_{\varepsilon}(y-\delta_n)]+\sup _y \Big\vert [\widetilde{F}_{n+1}(y+\delta_n)-\widetilde{F}_{n+1}(y-\delta_n)]\\
& \quad - [F_{\varepsilon}(y+\delta_n)-F_{\varepsilon}(y-\delta_n)] \Big\vert \\
&\stackrel{(ii)}{\leq} (L_{n+1}+1) \delta_n+2 \sup _y|\widetilde{F}_{n+1}(y)-F_{\varepsilon}(y)|,
\end{align*}
where $(i)$ is because $|\mathbf{1}\{a \leq y\}-\mathbf{1}\{b \leq y\}|\le \mathbf{1}\{|b-y|\le|a-b|\}$ for $a, b \in \R$ and $(ii)$ is because the Lipschitz continuity of $F_{\varepsilon}(y)$.
\end{proof}

\begin{lemma}
\label{t1}
Under assumption \ref{aw}-\ref{asm}, with probability higher than $1-Mn^{-\frac{2\gamma}{3}(\log_2 n+2)^{\frac{4}{3}}}$,
\begin{equation}
E\left(\sup_{y}|\tilde{F}_{n+1}(y)-F_{\varepsilon}(y)|^2 \right) \le  \frac{(2 +\log_2 n)^2}{n}(1+2\sqrt{M}g(b)+ n^{\frac{\gamma}{3}}(\log_2 n+2)^{-\frac{2}{3}}),
\end{equation}
where $\alpha = $ .
\end{lemma}

\begin{proof}
Define $Z_{n+1}(y) = \tilde{F}_{n+1}(y)-F_{\varepsilon}(y)$. Besides, we also define $Z_{n+1}(A) = \sum_{i=1}^{n}\omega_{n}(X(s_i))\mathbf{1}(X(s_i)\in A) - F_{\varepsilon}(\varepsilon\in A)$, where $A$ can be any region. Let $N$ be some positive integer to be chosen later. We first represent the CDF $F_\varepsilon(x)$ in base $2$:
\begin{equation}
F_\varepsilon(x)=\sum_{i=1}^N b_i(x)2^{-i}+r_N(x),
\end{equation}
where $r_N(x)\in [0,2^{-N})$ and $b_i=0$ or $1$.

For $l\in \{1,2,\cdots,N\}$, define
\begin{equation}
B_l(x)=\sum_{i=1}^l b_i(x)2^{-i}.
\end{equation}
We can define the points $x_i$ where $F_\varepsilon(x_i)=B_i(x)$. We have that
\begin{equation}
F_\varepsilon(x)-F_\varepsilon(x_i) = \sum_{i=l+1}^N a_i(x)2^{-i} + r_N(x) \le 2^{-l}.
\end{equation}
As a result, we can partition $Z_{n+1}(x)$ into the following sum:
\begin{equation}
\begin{aligned}
Z_{n+1}(F_{\varepsilon}^{-1}(x)) &= Z_{n+1}(F_{\varepsilon}^{-1}(B_1(x))) + \sum_{i=1}^{N-1}(Z_{n+1}(F_{\varepsilon}^{-1}(B_{i+1}(x)))-Z_{n+1}(F_{\varepsilon}^{-1}(B_i(x))))\\
&+(Z_{n+1}(F_{\varepsilon}^{-1}(y))-Z_{n+1}(F_{\varepsilon}^{-1}(B_{N}(x)))).    
\end{aligned}
\end{equation}
In order to bound $Z_{n+1}(y)$, we can bound each individual term instead. Since $B_{i+1}(x)-B_{i}(x)=b_{i+1}(x)2^{-(i+1)}\le 2^{-(i+1)}$, we know the interval $[B_{i}(x),B_{i+1}(x)]$ either has zero length, or it is equal to one region in the set $\{[(j-1)2^{-(i+1)},j2^{-(i+1)}],1\le j \le 2^{i+1}\}$. As a result, we have
\begin{equation}
\begin{aligned}
|Z_{n+1}(F_{\varepsilon}^{-1}(B_{i+1}(x)))-&Z_{n+1}(F_{\varepsilon}^{-1}(B_i(x)))|\le \\
&\sup_{j\in[1,2^{i+1}]}|Z_{n+1}(F_{\varepsilon}^{-1}(j2^{-(i+1)}))-Z_{n+1}(F_{\varepsilon}^{-1}((j-1)2^{-(i+1)}))|    
\end{aligned}
\end{equation}

Let \[\delta_i = \sup_{j\in[1,2^{i+1}]}|Z_{n+1}([F_{\varepsilon}^{-1}(j2^{-(i+1)}),F_{\varepsilon}^{-1}((j-1)2^{-(i+1)})])|,\] and \[\delta_{xN}=\sup_{x}|Z_{n+1}([F_{\varepsilon}^{-1}(B_N(x)),x])|.\] 

It follows that
\begin{equation}
|Z_{n+1}(F_{\varepsilon}^{-1}(y))| \le \sum_{i=1}^N \delta_i + \delta_{xN}.
\end{equation}
By the triangle inequality,
\begin{equation}
(E(\sup_{y\in[0,1]}|Z_{n+1}(F_{\varepsilon}^{-1}(y))|^2))^{1/2}\le \sum_{i=1}^N (E\delta_i^2)^{1/2} + (E\delta_{xN}^2)^{1/2}.
\end{equation}

Then we need to bound $\|\delta_i\|_2$ and $\|\delta_{xN}\|_2$ separately. Since $\delta_i$ is computing the supremum over a set, it can be bounded by the sum over the set,
\begin{equation}
\delta_i^2 \le \sum_{j=1}^{2^i}(Z_{n+1}(F_{\varepsilon}^{-1}(j2^{-(i+1)}))-Z_{n+1}(F_{\varepsilon}^{-1}((j-1)2^{-(i+1)}))^2.
\end{equation}

Taking expectation, we have
\begin{equation}
\begin{aligned}
E\delta_i^2 &\le \sum_{j=1}^{2^i}E(Z_{n+1}(F_{\varepsilon}^{-1}(j2^{-(i+1)}))-Z_{n+1}(F_{\varepsilon}^{-1}((j-1)2^{-(i+1)}))^2\\
& = \sum_{j=1}^{2^i} \operatorname{Var}(Z_{n+1}([F_{\varepsilon}^{-1}((j-1)2^{-(i+1)}),F_{\varepsilon}^{-1}(j2^{-(i+1)})])).
\end{aligned}
\end{equation}

Let $(\epsilon_i)_{i>0}$ be a sequence of independent and symmetric random variables in $\{-1,1\}$. For any finite partition $A_1,\cdots, A_k$ of $\mathbb{R}$, 
\begin{equation}
\begin{aligned}
\label{part}
\sum_{j=1}^k \operatorname{Var}Z_{n+1}(A_j) &= E(Z_{n+1}^2(\sum_{j=1}^{k}\epsilon_i\mathbf{1}_{A_i}))\\
&\stackrel{(i)}{\le} M_n^2(n+2\sum_{1\le i<j\le n}\alpha_{ij}),
\end{aligned}
\end{equation}
where $\alpha_{ij} = \alpha(\sigma(\varepsilon(s_i)),\sigma(\varepsilon(s_j)))$ is the alpha-mixing coefficient, $M_n=\max_{1\le i\le n} \omega_{n}(X(s_i))$ and $(i)$ is because of Lemma $1.1$ in \citep{rio2017asymptotic}.

Because of Assumption \ref{asm}, we have
\begin{equation}
\begin{aligned}
E\sum_{1\le i<j\le n}\alpha_{ij}&\le E\sum_{1\le i<j\le n}\alpha_1(|s_i-s_j|)g(b)\\
& \le n \sqrt{E\sum_{1\le i<j\le n}\alpha_1^2(|s_i-s_j|)}g(b)\\
& \le n^2 \sqrt{E_{d\sim g_n}\alpha_1^2(d)} g(b)\le n\sqrt{M}g(b),    
\end{aligned}
\end{equation}
where $g_n$ is the distribution of the distance between $s_i$ and $s_j$ for any $i,j\in\{1,\cdots,n\}$.

Besides, using Cauchy-Schwarz inequality, the variance can be bounded by
\begin{equation}
\begin{aligned}
\operatorname{Var}\sum_{1\le i<j\le n}\alpha_{ij} &= E\left(\sum_{1\le i<j\le n}\alpha_{ij}\right)^2 - \left(E\sum_{1\le i<j\le n}\alpha_{ij}\right)^2\\
&\le E\left(\sum_{1\le i<j\le n}\alpha_{ij}\right)^2\\
&\le n^2 E\left(\sum_{1\le i<j\le n}\alpha_{ij}^2\right)\\
&\le n^2 M.    
\end{aligned}
\end{equation}

From Markov inequality, we know that for any $k>0$,
\begin{equation}
\begin{aligned}
\label{im}
\mathbb{P}\left(\frac{1}{n}\left|\sum_{1\le i<j\le n}\alpha_{ij}-E\sum_{1\le i<j\le n}\alpha_{ij}\right|\ge k\right)&\le  \frac{\operatorname{Var}\sum_{1\le i<j\le n}\alpha_{ij}}{n^2 k^2}\\
&\le \frac{M}{k^2}.
\end{aligned}
\end{equation}

Let $k = n^{\frac{\gamma}{3}}(\log_2 n+2)^{-\frac{2}{3}}$, using Inequality \ref{im}, with probability higher than $1-Mn^{-\frac{2\gamma}{3}}(\log_2 n+2)^{\frac{4}{3}}$,
\begin{equation}
\label{alp}
\frac{1}{n}\sum_{1\le i<j\le n}\alpha_{ij} \le \sqrt{M}g(b) + n^{\frac{\gamma}{3}}(\log_2 n+2)^{-\frac{2}{3}}.
\end{equation}

Because $[F_{\varepsilon}^{-1}((j-1)2^{-(i+1)}), F_{\varepsilon}^{-1}(j2^{-(i+1)})]$ for $j=1,\cdots,2^{i+1}$ is a partition of $\mathbb{R}$, from \ref{part} and \ref{alp},
\begin{equation}
\begin{aligned}
E\delta_i^2 &\le \sum_{j=1}^{2^i} \operatorname{Var}(Z_{n+1}([F_{\varepsilon}^{-1}((j-1)2^{-(i+1)}),F_{\varepsilon}^{-1}(j2^{-(i+1)})]))\\
&\le nM_n^2(1+2\sqrt{M}g(b)+2 n^{\frac{\gamma}{3}}(\log_2 n+2)^{-\frac{2}{3}}).
\end{aligned}
\end{equation}

For the other term $\delta_{xN}$, we know $x= F_{\varepsilon}^{-1}(F_\varepsilon(x))\le F_{\varepsilon}^{-1}(B_N(x)+r_N(x))$. We have
\begin{equation}
\begin{aligned}
Z_{n+1}([F_{\varepsilon}^{-1}(B_N(x)),x]) &=\tilde{F}_{n+1}([F_{\varepsilon}^{-1}(B_N(x)),x]) - F_\varepsilon([F_{\varepsilon}^{-1}(B_N(x)),x])\\
&\ge - F_\varepsilon([F_{\varepsilon}^{-1}(B_N(x)),x])\\
& = B_N(x)-x \ge -2^{-N}.
\end{aligned}
\end{equation}

On the other hand,
\begin{equation}
\begin{aligned}
Z_{n+1}([F_{\varepsilon}^{-1}(B_N(x)),x]) &= Z_{n+1}([F_{\varepsilon}^{-1}(B_N(x)),B_N(x))+2^{-N}]) - Z_{n+1}([F_{\varepsilon}^{-1}(x,B_N(x))+2^{-N}])\\
&\le Z_{n+1}([F_{\varepsilon}^{-1}(B_N(x)),B_N(x))+2^{-N}]) + 2^{-N}.    
\end{aligned}
\end{equation}

As a result, we have
\begin{equation}
\delta_{xN}\le \delta_{N} + 2^{-N}.
\end{equation}

To sum up, we prove that
\begin{equation}
\begin{aligned}
(E(\sup_{y\in[0,1]}|Z_{n+1}(F_{\varepsilon}^{-1}(y))|^2))^{\frac{1}{2}}\le n^{\frac{1}{2}}M_n(N+1+2^{-N})(1+2\sqrt{M}g(b)+ n^{\frac{\gamma}{3}}(\log_2 n+2)^{-\frac{2}{3}})^{\frac{1}{2}}.
\end{aligned}
\end{equation}

Let $N=\log_2 n$, we have
\begin{equation}
E(\sup_{y\in[0,1]}|Z_{n+1}(F_{\varepsilon}^{-1}(y))|^2)\le M_n^2 n(2 +\log_2 n)^2(1+2\sqrt{M}g(b)+ n^{\frac{\gamma}{3}}(\log_2 n+2)^{-\frac{2}{3}}).
\end{equation}
\end{proof}

\begin{corollary}
Under Assumptions \ref{aw}-\ref{asm}, with probability higher than $1-(2+M+2\sqrt{M}g(b))n^{-\frac{2\gamma}{3}}$\\
$(\log_2 n+2)^{\frac{4}{3}}$,
\begin{equation}
\sup_y\left|\widetilde{F}_{n+1}(y)-F_\epsilon(y)\right|\le M_n n^{\frac{1+\gamma}{2}},
\end{equation}    
where $\widetilde{F}_{n+1}(y)=\sum_{i=1}^{n}\omega_{n}(X(s_i))\mathbf{1}(\epsilon(s_i)\le y)$.
\end{corollary}

\begin{proof}
We have that
\begin{equation}
\begin{aligned}
\tilde{F}_{n+1}(y)-F_\epsilon(y)&=\sum_{i=1}^n\omega_{n}(X(s_i))\mathbf{1}(\epsilon(s_i)\le y)-F_\epsilon(y)\\
&=\sum_{i=1}^n\omega_{n}(X(s_i))(\mathbf{1}(\epsilon(s_i)\le y)-F_\epsilon(y)).
\end{aligned}
\end{equation}
Let $Z(s_i) = \mathbf{1}(\epsilon(s_i)\le y)-F_\epsilon(y)$, we know
\begin{equation}
E Z(s_i) = 0,
\end{equation}
and $Z(s)$ is stationary. From Lemma \ref{t1}, using Markov inequality, we have that for any $k$,
\begin{equation}
\begin{aligned}
\mathbb{P}(\sup_y |\tilde{F}_{n+1}(y)-F_\varepsilon(y)|\ge k)&\le \frac{E(\sup_y |\tilde{F}_{n+1}(y)-F_\varepsilon(y)|^2)}{k^2}\\
&=\frac{M_n^2 n(2+\log_2 n)^2(1+2\sqrt{M}g(b)+n^{\frac{\gamma}{3}}(\log_2 n+2)^{-\frac{2}{3}})}{k^2}.    
\end{aligned}
\end{equation}

Let $k = M_n n^{\frac{1+\gamma}{2}}$, we have
\begin{equation}
\begin{aligned}
\mathbb{P}\left(\sup_y |\tilde{F}_{n+1}(y)-F_\varepsilon(y)|\ge M_n n^{\frac{1+\gamma}{2}}\right)&\le (2+\log_2 n)^{2}n^{-\gamma}(1+2\sqrt{M}g(b)+n^{\frac{\gamma}{3}}(\log_2 n+2)^{-\frac{2}{3}})\\
&= (1+2\sqrt{M}g(b))(2+\log_2 n)^{2}n^{-\gamma} + (2+\log_2 n)^{\frac{4}{3}}n^{-\frac{2\gamma}{3}}\\
&\le (2+2\sqrt{M}g(b))(2+\log_2 n)^{\frac{4}{3}}n^{-\frac{2\gamma}{3}}.
\end{aligned}
\end{equation}
The last inequality holds because the order of $log_2 n$ is smaller than $n^{\frac{\gamma}{3}}$.
\end{proof}

\begin{theorem}
Under Assumption \ref{aw}-\ref{asm}, for any $\alpha\in(0,1)$ and sample size $T$, we have
\begin{equation}
\begin{aligned}
&\left|\mathbb{P}\left(Y(s_{n+1}) \in \widehat{C}_{t-1}\left(X(s_{n+1})\right) \mid X(s_{n+1})=X(s_{n+1})\right)-(1-\alpha)\right| \\
&\le 4L_{n+1}\delta_n+ 6M_n n^{\frac{1+\gamma}{2}}+ (4+2M+4\sqrt{M}g(b))n^{-\frac{2\gamma}{3}}(\log_2 n+2)^{\frac{4}{3}}. 
\end{aligned}
\end{equation}
\end{theorem}

\begin{proof}
For simplicity, we use $X_{i} = X(s_i)$, $Y_{i} = Y(s_i)$, $\varepsilon_i = \varepsilon(s_i)$ and $\omega_{ni} = \omega_{n}(X(s_i))$. For any $\beta\in[0,1]$,
\begin{equation}
\begin{aligned}
&\left|\mathbb{P}\left(Y_{n+1} \in \widehat{C}_{t-1}\left(X_{n+1}\right) \mid X_{n+1}\right) - (1-\alpha)\right| \\
&= \left|\mathbb{P}\left(\varepsilon_{n+1} \in [\widehat{Q}_{\beta^*}(X_{n+1}), \widehat{Q}_{1-\alpha+\beta^*}(X_{n+1})] \mid X_{n+1}\right) - (1-\alpha)\right| \\
&= \left|\mathbb{P}\left(\beta \leq \sum_{i=1}^n \omega_{ni}\mathbf{1}(\hat{\varepsilon}_i \leq \hat{\varepsilon}_{n+1}) \leq 1-\alpha+\beta\right) - (1-\alpha)\right| \\
&= \left|\mathbb{P}\left(\beta \leq \widehat{F}_{n+1}(\hat{\varepsilon}_{n+1}) \leq 1-\alpha+\beta\right) - \mathbb{P}\left(\beta \leq F_{\varepsilon}(\varepsilon_{n+1}) \leq 1-\alpha+\beta\right)\right| \\
&\leq \mathbb{E}\left|\mathbf{1}\{\beta \leq \widehat{F}_{n+1}(\hat{\varepsilon}_{n+1}) \leq 1-\alpha+\beta\} - \mathbf{1}\{\beta \leq F_{\varepsilon}(\varepsilon_{n+1}) \leq 1-\alpha+\beta\}\right| \\
&\stackrel{(i)}{\leq} \mathbb{E}\Big(\left|\mathbf{1}\{\beta \leq \widehat{F}_{n+1}(\hat{\varepsilon}_{n+1})\} - \mathbf{1}\{\beta \leq F_{\varepsilon}(\varepsilon_{n+1})\} \right| \\
&\qquad + \left|\mathbf{1}\{\widehat{F}_{n+1}(\hat{\varepsilon}_{n+1}) \leq 1-\alpha+\beta\} - \mathbf{1}\{F_{\varepsilon}(\varepsilon_{n+1}) \leq 1-\alpha+\beta\} \right|\Big) \\
&\stackrel{(ii)}{\leq} \mathbb{P}\left(\left|F_\varepsilon(\varepsilon_{n+1}) - \beta\right| \leq \left|F_\varepsilon(\varepsilon_{n+1}) - \widehat{F}_{n+1}(\hat{\varepsilon}_{n+1})\right|\right) \\
&\qquad + \mathbb{P}\left(\left|F_\varepsilon(\varepsilon_{n+1}) - (1-\alpha+\beta)\right| \leq \left|F_\varepsilon(\varepsilon_{n+1}) - \widehat{F}_{n+1}(\hat{\varepsilon}_{n+1})\right|\right),
\end{aligned}
\end{equation}

where Inequality $(i)$ follows since for any constants $a, b$ and univariates $x, y, \mid \mathbf{1}\{a \leq x \leq$ $b\}-\mathbf{1}\{a \leq y \leq b\}|\leq|\mathbf{1}\{a \leq x\}-\mathbf{1}\{a \leq y\}|+|\mathbf{1}\{x \leq b\}-\mathbf{1}\{y \leq b\}|$. On the other hand, Inequality $(ii)$ is a result of $|\mathbf{1}\{a \leq x\}-\mathbf{1}\{b \leq x\}| \leq \mathbf{1}\{|b-x| \leq|a-b|\}$. Using Lemma \ref{t1},
\begin{equation}
\begin{aligned}
&\mathbb{P}\left(\left|F_\varepsilon(\varepsilon_{n+1})-\beta\right|\le \left|F_\varepsilon(\varepsilon_{n+1})-\widehat{F}_{n+1}(\hat{\varepsilon}_{n+1})\right|\right)\\
&\le \mathbb{P}\left(\left|F_\varepsilon(\varepsilon_{n+1})-\beta\right|\le \left|F_\varepsilon(\varepsilon_{n+1})-\widehat{F}_{n+1}(\hat{\varepsilon}_{n+1})\right|, \sup_y \left|F_\varepsilon(y)-\widehat{F}_{n+1}(y)\right|\le M_n n^{\frac{1+\gamma}{2}}\right)\\
&+\mathbb{P}\left(\sup_y \left|F_\varepsilon(y)-\widehat{F}_{n+1}(y)\right|\ge M_n n^{\frac{1+\gamma}{2}}\right)\\
&\le\mathbb{P}\left(\left|F_\varepsilon(\varepsilon_{n+1})-\beta\right|\le \left|F_\varepsilon(\varepsilon_{n+1})-\widehat{F}_{n+1}(\hat{\varepsilon}_{n+1})\right|\mid \sup_y \left|F_\varepsilon(y)-\widehat{F}_{n+1}(y)\right|\le M_n n^{\frac{1+\gamma}{2}}\right)\\&
+ (2+M+2\sqrt{M}g(b))n^{-\frac{2\gamma}{3}}(\log_2 n+2)^{\frac{4}{3}}\\
& \le \mathbb{P}\left(\left|F_\varepsilon(\varepsilon_{n+1})-\beta\right|\le \left|F_\varepsilon(\varepsilon_{n+1})-F_\varepsilon(\hat{\varepsilon}_{n+1})\right|+ (L_{n+1}+1)\delta_n + 3M_n n^{\frac{1+\gamma}{2}}\right) \\
&+ (2+M+2\sqrt{M}g(b))n^{-\frac{2\gamma}{3}}(\log_2 n+2)^{\frac{4}{3}}\\
&\le \mathbb{P}\left(\left|F_\varepsilon(\varepsilon_{n+1})-\beta\right|\le L_{n+1}\left|\varepsilon_{n+1}-\hat{\varepsilon}_{n+1}\right|+ (L_{n+1}+1)\delta_n + 3M_n n^{\frac{1+\gamma}{2}}\right)\\
&+ (2+M+2\sqrt{M}g(b))n^{-\frac{2\gamma}{3}}(\log_2 n+2)^{\frac{4}{3}}\\
&\le 2L_{n+1}\delta_n+ 3M_n n^{\frac{1+\gamma}{2}}+ (2+M+2\sqrt{M}g(b))n^{-\frac{2\gamma}{3}}(\log_2 n+2)^{\frac{4}{3}}.
\end{aligned}
\end{equation}

The inequality above also holds for $\mathbb{P}\left(\left|F_\varepsilon(\varepsilon_{n+1})-\beta\right|\le \left|F_\varepsilon(\varepsilon_{n+1})-\widehat{F}_{n+1}(\hat{\varepsilon}_{n+1})\right|\right)$, so
\begin{equation}
\begin{aligned}
&\left|\mathbb{P}\left(Y_{n+1} \in \widehat{C}_{t-1}\left(X_{n+1}\right) \mid X_{n+1}\right)-(1-\alpha)\right|\\
&\le 4L_{n+1}\delta_n+ 6M_n n^{\frac{1+\gamma}{2}}+ (4+2M+4\sqrt{M}g(b))n^{-\frac{2\gamma}{3}}(\log_2 n+2)^{\frac{4}{3}}.
\end{aligned}
\end{equation}
\end{proof}
\end{document}